\documentclass{article}
\usepackage[round]{natbib}
\usepackage{mathrsfs}
\usepackage{latexsym}
\usepackage{amsmath}
\usepackage{amsfonts}
\usepackage{url}
\usepackage{amssymb}
\usepackage{amsthm}

\usepackage{graphicx} %
\usepackage{subfigure} 
\usepackage{epstopdf}
\usepackage{float}

\usepackage{tikz}
\usetikzlibrary{positioning}
\usepackage{epstopdf}
\usepackage{hyperref}

\usepackage{natbib}

\renewcommand{\b}[1]{\boldsymbol{#1}}
\newcommand{\bet}{\b{\eta}}
\newcommand{\wt}{\tilde}
\newcommand{\wh}{\widehat}

\newcommand{\mc}[1]{\mathcal{#1}}
\newcommand{\D}{P}
\renewcommand{\H}{\mathcal{H}}
\newcommand{\radem}{\mc{R}}
\newcommand{\dist}{d} %
\newcommand{\distf}{\rho}
\newcommand{\marg}{\gamma}

\mathchardef\mhyphen="2D
\newcommand{\trun}[2]{{\operatorname{T}}_{[{#1},{#2}]}}
\newcommand{\TPhF}{\H}  %
\newcommand{\lipscore}{F}  %
\newcommand{\samp}{S}

\newcommand{\mist}{m}
\newcommand{\loss}{\mathcal{L}}
\newcommand{\norm}[1]{\| #1 \|}
\newcommand{\lp}{\left(}
\newcommand{\rp}{\right)}

\renewcommand{\a}{\alpha}
\newcommand{\e}{\eps}

\newcommand{\Lip}[1]{\nrm{#1}_{\textrm{{\tiny \textup{Lip}}}}}
\newcommand{\Lipts}[1]{\tsnrm{#1}_{\textrm{{\tiny \textup{Lip}}}}}
\newcommand{\ddim}{\operatorname{ddim}}

\newcommand{\argmin}{\mathop{\mathrm{argmin}}}
\newcommand{\argmax}{\mathop{\mathrm{argmax}}}

\newcommand{\natr}{_{\textrm{{\tiny \textup{Nat}}}}}
\newcommand{\snatr}{_{
\marg
\textrm{{\tiny \textup{Nat}}}
}}
\newcommand{\vc}{_{\textrm{{\tiny \textup{VC}}}}}

\newcommand{\cutoff}{_{\textrm{{\tiny \textup{cutoff}}}}}
\newcommand{\margin}{_{\textrm{{\tiny \textup{margin}}}}}
\newcommand{\Rad}{_{\textrm{{\tiny \textup{Rad}}}}}

\newcommand{\fnn}{f_{\textrm{{\tiny \textup{NN}}}}}
\newcommand{\gnn}{g_{\textrm{{\tiny \textup{NN}}}}}

\newcommand{\tmoo}{\trun{\mhyphen1}{1}}

\newcommand{\X}{\calX}
\newcommand{\Y}{\mathcal{Y}}

\renewcommand{\H}{\mathcal{H}}
\newcommand{\F}{\mathcal{F}}

\usepackage{dsfont}
\newcommand{\chr}{\mathds{1}}
\newcommand{\pred}[1]{\chr_{\left\{ #1 \right\}}}

\newcommand{\E}{\mathbb{E}}

\newcommand{\fat}{\mathrm{fat}}

\newcommand{\diam}{\operatorname{diam}}

\renewcommand{\P}{\mathbb{P}}

\newcommand{\ben}{\begin{enumerate}}
\newcommand{\een}{\end{enumerate}}
\newcommand{\bit}{\begin{itemize}}
\newcommand{\eit}{\end{itemize}}

\def\clap#1{\hbox to 0pt{\hss#1\hss}}

\newcommand{\tsabs}[1]{| #1 |}
\newcommand{\nrm}[1]{\left\Vert #1 \right\Vert}

\newcommand{\tsnrm}[1]{\Vert #1 \Vert}

\newcommand{\pl}[1]{\paren{#1}_{\!+}}

\newcommand{\calX}{\mathcal{X}}

\newcommand{\R}{\mathbb{R}}
\newcommand{\N}{\mathbb{N}}

\newcommand{\beq}{\begin{eqnarray*}}
\newcommand{\eeq}{\end{eqnarray*}}
\newcommand{\beqn}{\begin{eqnarray}}
\newcommand{\eeqn}{\end{eqnarray}}
\newcommand{\paren}[1]{\left( #1 \right)}
\newcommand{\sqprn}[1]{\left[ #1 \right]}
\newcommand{\tlprn}[1]{\left\{ #1 \right\}}
\newcommand{\set}[1]{\tlprn{#1}}
\newcommand{\abs}[1]{\left| #1 \right|}
\newcommand{\ceil}[1]{\ensuremath{\left\lceil#1\right\rceil}}

\newcommand{\gn}{\, | \,}
\newcommand{\ds}{\displaystyle}

\newcommand{\hide}[1]{}
\newcommand{\oo}[1]{\frac{1}{#1}}
\def\eps{\varepsilon}

\newtheorem{theorem}{Theorem}

\newtheorem{lemma}[theorem]{Lemma}

\newcommand{\bepf}{\begin{proof}}
\newcommand{\enpf}{\end{proof}}

\newcommand{\veta}{\b{\eta}}
\def\eps{\varepsilon}

\title{
Maximum Margin Multiclass Nearest Neighbors
}

\author{
Aryeh Kontorovich
\and
Roi Weiss
}

\begin{document}
\maketitle
\begin{abstract}
We develop a general framework for margin-based multicategory classification in metric spaces.
The basic work-horse is a margin-regularized version of the nearest-neighbor classifier.
We prove
generalization bounds
that 
match the state of the art in sample size $n$ and significantly 
improve the dependence on the number of classes $k$.
Our point of departure is
a nearly Bayes-optimal 
finite-sample
risk bound 
independent of $k$.
Although $k$-free, 
this bound is unregularized and non-adaptive,
which motivates our main result:
Rademacher and scale-sensitive 
margin
bounds with a logarithmic dependence on $k$.
As the best previous risk estimates 
in this setting
were of order $\sqrt k$, 
our bound is exponentially sharper.
From the algorithmic standpoint,
in doubling metric spaces
our classifier may be trained
on $n$ examples in 
$O(n^2\log n)$ 
time and evaluated on new points in $O(\log n)$ time.
\end{abstract} 

\section{Introduction}
Whereas the theory of supervised binary classification is by now fairly well 
developed,
its multiclass extension 
continues to pose
numerous novel statistical and computational challenges.
On the algorithmic front, there is the basic question of
how to adapt the hyperplane and kernel methods ---  
ideally suited for
two classes --- to three or more.
A host of new problems also arises on the statistical front. In the binary case,
the VC-dimension characterizes the distribution-free sample complexity \citep{MR1741038} and tighter distribution-dependent
bounds are available via Rademacher techniques \citep{DBLP:journals/jmlr/BartlettM02,MR1892654}. 
Characterizing the multiclass distribution-free sample complexity is far less straightforward, though impressive progress
has been recently made \citep{DBLP:journals/jmlr/DanielySBS11}.

Following \citet{DBLP:journals/jmlr/LuxburgB04,DBLP:conf/colt/GottliebKK10},
we
adopt a proximity-based approach to supervised multicategory classification in metric spaces.
The principal motivation for this framework is two-fold:
\begin{itemize}
\item[(i)] Many natural metrics, such as $L_1$, earthmover, and edit distance cannot be embedded in a Hilbert space without a large distortion
\citep{Enflo69,NS07,AK10}.
Any kernel method is thus a priori at a disadvantage 
when learning to classify non-Hilbertian objects,
since it cannot faithfully represent the data geometry.
\item[(ii)] Nearest neighbor-based classification 
sidesteps the issue of
$k$-to-binary reductions --- which, despite voluminous research,
is still the subject of vigorous debate \citep{DBLP:journals/jmlr/RifkinK03,ElYaniv20081954}.
In terms of time complexity,
the reductions approach faces an
$\Omega(k)$ information-theoretic lower bound
\citep{DBLP:conf/alt/BeygelzimerLR09},
while nearest neighbors admit solutions whose runtime does not depend on
the number of classes.
\end{itemize}
\paragraph{Main results.} Our contributions are both statistical and  
algorithmic in nature. On the statistical front, 
we open with the observation that the 
nearest-neighbor classifier's
expected risk is at most twice the Bayes optimal plus a term that decays with
sample size at a rate {\em not dependent} on the number of classes $k$ 
(and continues to hold for $k=\infty$, Theorem~\ref{thm:bayes}).
Although of interest as 
apparently
the first ``$k$-free'' finite-sample result,
it has the drawback of being 
{\em non-adaptive} in the sense of depending on 
properties of the unknown 
sampling distribution
and failing to provide the learner with a usable data-dependent bound.
This difficulty is overcome in our
main technical contribution (Theorems~\ref{thm:radem} and \ref{thm:scale}),
where we give a margin-based multiclass bound of 
order 
\beqn
\label{eq:intro_bound}
\min \left\{  \oo{\marg}\lp \frac{\log k}{n} \rp^{\frac{1}{D+1}},  \oo{\marg^{\frac{D}{2}}}  \lp \frac{\log k}{n} \rp^{\oo{2}} \right\},
\eeqn
where $k$ is the number of classes, $n$ is sample size, $D$ is the doubling dimension of the metric instance space and $0<\marg\le1$ is the margin.
This matches the state of the art asymptotics in $n$ for metric spaces and significantly improves the
dependence on $k$, which 
hitherto was of order $\sqrt k$ \citep{zhang2002covering, zhang2004statistical} or worse.
The exponential dependence on 
some covering dimension (such as $D$)
is in general inevitable, as shown by a standard no-free-lunch argument
\citep{shai2014}, but whether (\ref{eq:intro_bound}) is optimal remains an open question.

On the algorithmic front, using the above bounds, we show how to efficiently perform 
Structural Risk Minimization (SRM) 
so as to avoid
overfitting. 
This involves deciding {\em how many} and {\em which} sample points one is allowed to err on.
We reduce this problem
to minimal
vertex cover, which admits
a greedy $2$-approximation. 
Our algorithm
admits a significantly faster $\eps$-approximate version  
in doubling spaces
with a graceful degradation 
in $\eps$
of the generalization bounds,
based on approximate nearest neighbor techniques developed by 
\citet{DBLP:conf/colt/GottliebKK10,DBLP:conf/simbad/GottliebKK13}.
For a fixed doubling dimension and $\eps$,
our runtime 
is
$O(n^2\log n)$
for 
learning
and 
$O(\log n)$ for evaluation on a test point.
(Exact nearest neighbor requires $\Theta(n)$ evaluation time.)
Finally, our generalization bounds and algorithm can be made adaptive to the intrinsic dimension
of the data 
via a recent metric dimensionality-reduction technique \citep{gkr2013-alt}.
\paragraph{Related work.}
Due to space constraints, we are only able to mention the most directly relevant results --- 
and even these, not in full generality but rather with an eye to facilitating comparison to the present work.
Supervised $k$-category classification approaches follow two basic paradigms:
{\bf(I)} defining a score function on point-label pairs and classifying by choosing the label with the 
optimal score and
{\bf(II)} reducing the problem to several binary classification problems.
Regarding the second paradigm,
the seminal paper of \citet{allwein2001reducing} unified the various error correcting
output code (ECOC)-based multiclass-to-binary reductions under a single margin-based framework.
Their generalization bound requires the base classifier to have VC-dimension $d\vc<\infty$ 
(and hence does not apply to nearest neighbors or infinite-dimensional Hilbert spaces) and is of the form 
$\wt{O}\Big(\frac{\log k}{\marg}\sqrt{\frac{d\vc}{n}}\Big)$.
\citet{DBLP:conf/colt/LangfordB05,DBLP:conf/alt/BeygelzimerLR09} gave 
$k$-free and $O(\log k)$
regret bounds,
but these
are conditional on the performance of the underlying binary classifiers 
as opposed to the unconditional bounds we provide in this paper.

As for the first paradigm, proximity is perhaps the most natural score function ---
and indeed, a formal analysis of the 
nearest neighbor classifier \citep{cover1967nearest} much predated the first multiclass extensions of 
SVM \citep{weston1999support}.
\citet{crammer2002algorithmic,crammer2002learnability} considerably reduced the computational complexity of the 
latter approach and gave a risk bound decaying as $\wt{O}(k^2/n\marg^2)$, 
for the separable case with margin $\marg$.
In an alternative approach based on choosing $q$ prototype examples, 
\citet{crammer2002margin} gave a risk bound with rate $\wt{O}(q^{k/2}/\marg\sqrt n)$.
\citet{MR1322634} characterized the PAC learnability of $k$-valued functions in
terms of combinatorial dimensions, such as the Natarajan dimension $d\natr$.
\citet{MR2383568,MR2745296} gave scale-sensitive analogues of these dimensions.
He gave a risk bound decaying as
$\wt{O}\paren{\frac{\log k}{\marg}\sqrt {d\snatr/n}}$,
where $d\snatr$ is a scale-sensitive Natarajan dimension ---
essentially replacing the finite VC dimension $d\vc$ in 
\citet{allwein2001reducing} by $d\snatr$. 
He further showed that for linear function classes in Hilbert spaces, $d\snatr$ is bounded by $\wt{O} ( k ^ 2/ \marg^2)$, 
resulting in a risk bound decaying as $\wt{O}({k}/{\marg^2\sqrt{n}})$.
To the best of our knowledge,
the sharpest current estimate on the Natarajan dimension
(for some special function classes) 
is $d\natr = \wt{O}(k)$ 
with a matching lower bound of $\Omega(k)$ \citep{DBLP:journals/jmlr/DanielySBS11}.
A margin-based Rademacher analysis 
of score functions
\citep{mohri-book2012}
yields a bound of order $\wt{O}(k^2/\marg\sqrt n)$,
and this is also the $k$-dependence obtained by
\citet{icml2013_cortes13} 
in a recent paper
proposing a multiple kernel approach to multiclass learning.
Closest in spirit to our work are the results of
\citet{zhang2002covering,zhang2004statistical},
who used the chaining technique to achieve a Rademacher complexity
with asymptotics
$\wt{O}\Big(\frac{1}{\marg}\sqrt{\frac{k}{n}}\Big)$. 

Besides the 
dichotomy of
score functions vs. 
multiclass-to-binary reductions 
outlined above, 
multicategory risk bounds 
may also be grouped
by
the trichotomy
of
{\bf(a)} combinatorial dimensions {\bf(b)} Hilbert spaces {\bf(c)} metric spaces
(see Table~\ref{table:comp_bounds}).
Category (a)
is comprised of algorithm-independent results that
give generalization bounds in terms of some combinatorial dimension of a fixed concept class
\citep{allwein2001reducing,MR1322634,MR2383568,MR2745296,DBLP:journals/jmlr/DanielySBS11}.
Multiclass extensions of SVM and related kernel methods
\citep{weston1999support,crammer2002algorithmic,crammer2002learnability,crammer2002margin,icml2013_cortes13} 
fall into category (b).
Category (c), consisting of agnostic\footnote{
in the sense of not requiring an a priori fixed concept class
} metric-space methods is the most sparsely populated.
The pioneering asymptotic analysis of 
\citet{cover1967nearest} was cast in a
modern, finite-sample version by 
\citet{shai2014}, but only for binary classification.
Unlike Hilbert spaces, which admit dimension-free margin bounds,
we are not aware of any metric space risk bound that does not explicitly depend on some metric
dimension $D$ or covering numbers. The bounds in
\citet{shai2014,gkr2013-alt}
exhibit
a characteristic ``curse of dimensionality''
decay rate of $O(n^{-1/(D+1)})$,
but more optimistic asymptotics can be obtained
\citep{MR2383568,MR2745296,zhang2002covering,zhang2004statistical,DBLP:conf/colt/GottliebKK10}.
Although some sample lower bounds 
for proximity-based methods
are known \citep{shai2014}, 
the optimal dependence on $D$ and $k$
is far from being fully understood.

\begin{table}
\begin{tabular}{l|l|l}
\rule{0pt}{2ex} Paper                 &\hspace{-2pt}decay rate \hspace{-2pt}\scriptsize{ $\wt{O}(\cdot)$}\hspace{-3pt}&\hspace{-2pt}group\hspace{-3pt}\\
[2pt] \hline
\rule{0pt}{3ex}\citet{allwein2001reducing}$^\ddagger$ & $\frac{\log k}{\marg}\sqrt{\frac{d\vc}{n}}$ & (II,a) \\
\rule{0pt}{3ex}\citet{DBLP:journals/jmlr/DanielySBS11}$^{*\dagger\ddagger}$ & $ \frac{d\natr \log k }{n}$ & (I,a)\\
\rule{0pt}{3ex}\citet{MR2745296}$^\ddagger$    & $ \frac{\log k}{\marg}\sqrt{\frac{d\snatr }{n} } $ & (I,a)\\
\rule{0pt}{3ex}\citet{crammer2002learnability}$^\dagger$ & $\frac{k^2}{\marg^2 n}$ & (I,b) \\
\rule{0pt}{3ex}\citet{icml2013_cortes13} & $ \frac{k^2}{\marg\sqrt n} $ & (I,b)\\
\rule{0pt}{3ex}\citet{MR2745296}    & $ \frac{k}{\marg^2\sqrt n}$ & (I,b)\\
\rule{0pt}{3ex}\citet{zhang2004statistical} & $ \oo{\marg}\sqrt{\frac{k}{n}}$ & (I,b)\\
\rule{0pt}{3ex}current paper & $\oo{\marg^{D/2}} \sqrt\frac{\log k}{n}$ & (I,c)\\
\rule{0pt}{3ex}current paper & $\oo{\marg}{\lp\frac{\log k}{n}\rp^{\oo{1+D}}}$ & (I,c)\\
[5pt]                   
\end{tabular}
\caption{Comparing various multiclass bounds. {\scriptsize $(^*)$ Not margin-based. $(^\dagger)$ Only for the separable case.
$(^\ddagger)$ Combinatorial dimension depends on $k$.}
}
\label{table:comp_bounds}
\end{table}
\section{Preliminaries}
\label{sec:prelim}
\paragraph{Metric Spaces.} Given two metric spaces $(\X,\dist)$ and $(\mc{Z},\distf)$,
a function $f: \X \to \mc{Z}$ is called $L$-Lipschitz if
\(
\distf(f(x),f(x')) \leq L \dist(x,x')
\)
for all $x,x'\in\X$. 
(The real line $\R$ is always considered with its Euclidean metric $\abs{\cdot}$.)
The Lipschitz  constant of $f$, denoted $\Lip{f}$, 
is the smallest $L$ for which $f$ is $L$-Lipschitz.
The distance between two sets $A,B\subset\X$ is defined by $\dist(A,B)=\inf_{x\in A,x'\in B}\dist(x,x')$.
For a metric space $(\X, \dist)$, let $\lambda$ be the smallest value such that every ball in $\X$ can be
covered by $\lambda$ balls of half the radius. 
The {\em doubling dimension} of $\X$ is $\ddim(\X):=\log_2 \lambda$. A metric is
{\em doubling} when its doubling dimension is bounded.
The $\e$-covering number of a metric space $(\X,\dist)$, denoted $\mc{N}(\e,\X,\dist)$, is
defined as the smallest number of balls of radius $\e$ that suffices to cover $\X$.
It can be shown 
(e.g., \citet{KL04})
that 
\beqn
\label{eq:X_cov_Num}
\mc{N}(\e,\X,\dist) \leq \lp \frac{2\diam(\X)}{\e} \rp^{\ddim(\X)},
\eeqn
where ${\ds \diam(\X)=\sup_{x,x'\in\X}\dist(x,x')}$ is the diameter of $\X$.

\paragraph{The multiclass learning framework.} 
Let $(\X,\dist)$ be a metric instance space with $\diam(\X) = 1$,
$\ddim(\X)=D<\infty$,
 and $\Y\subseteq\N$ an at most
countable {label} set.
We observe a sample
$\samp=\paren{X_i,Y_i}_{i=1}^n \in \{\X\times \Y\}^n$ drawn iid from an unknown distribution $\D$ 
over $\X\times\Y$.

In line with paradigm (I) outlined in the Introduction, our classification procedure consists of 
optimizing a score function. In hindsight, the score at a 
test
point will be determined by its labeled neighbors,
but for now, we consider an unspecified collection $\F$ of functions mapping $\X\times\Y$ to $\R$.
A score function $f\in\F$ induces the classifier $g_f:\X\to\Y$
 via
\beqn
\label{eq:g_f}
g_f(x) = \argmax_{y\in \Y} f(x,y),
\eeqn
breaking ties arbitrarily. 
The {\em margin} of $f\in\F$ on $(x,y)$ is defined by
\beqn
\label{eq:margin_fun}
\marg_f(x,y) = \oo{2} \lp f(x,y) - \sup_{y' \neq y} f(x,y') \rp.
\eeqn
Note that $g_f$ 
misclassifies $(x,y)$ precisely when $\marg_f(x,y)<0$. 
One of our main objectives is to upper-bound
the generalization error 
\beq
\P(g_f(X) \neq Y) = \E[\pred{\marg_f(X,Y)<0}].
\eeq
To this end, we introduce two surrogate
loss functions $\loss:\R\to\R_+$:
\beq
\loss\cutoff(u) &=& \pred{u<1} \\
\loss\margin(u) &=& \trun{0}{1}(1-u),
\eeq
where
\beqn
\label{eq:T}
\trun{a}{b}(z) = \max\set{a,\min\set{b,z}}
\eeqn
is the truncation operator.
The empirical loss $\wh\E[\loss(\marg_f)]$ induced by 
any of the loss functions above is
\(
\frac{1}{n} \sum_{i=1}^n \loss(\marg_f(X_i,Y_i)).
\)
All probabilities $\P(\cdot)$ and expectations $\E[\cdot]$ are with
respect to the sampling distribution $\D$. We will write $\E_\samp$
to indicate expectation over a sample (i.e., over $\D^n$).

\section{Risk bounds}
In this section we analyze the statistical properties of 
nearest-neighbor
multicategory classifiers in metric spaces.
In Section~\ref{sec:bayes}, Theorem \ref{thm:bayes},
we 
record the observation that 
the 1-nearest neighbor classifier is
nearly
Bayes optimal,
with a risk decay that does not depend on the number of classes $k$. 
Of course, the 1-naive nearest neighbor is well-known to overfit.
This is reflected in the non-adaptive nature of the analysis:
the bound is stated in terms of properties of the unknown sampling distribution, 
and fails to provide the learner with a usable data-dependent bound.

To achieve the latter goal, we develop 
a margin analysis
in Section~\ref{sec:margin}.
Our main technical result
is Lemma~\ref{lem:covNum},
from which the logarithmic dependence on $k$ claimed in (\ref{eq:intro_bound}) follows.
Although not $k$-free like the Bayes excess risk bound of Theorem~\ref{thm:bayes},
$O(\log k)$ is
exponentially sharper than the current state of the art
\citep{zhang2002covering, zhang2004statistical}.
Whether a $k$-free metric entropy bound is possible is currently left as an open problem.

The metric entropy bound of Lemma~\ref{lem:covNum} facilitates two approaches to bounding the risk:
via Rademacher complexity (Section~\ref{sec:rademacher}) and via scale-sensitive techniques in the spirit of \citet{MR2383568}
(Section~\ref{sec:scale-sensitive}).
In Section~\ref{sec:main-theorem} we combine these two margin bounds by taking their minimum. The resulting bound will be used in Section~\ref{sec:alg} to perform efficient Structural Risk Minimization.
\subsection{Multiclass Bayes near-optimality}
\label{sec:bayes}
In this section, $(\X,\dist)$ is a metric space
and $\Y$ is an at most countable (possibly infinite) label set.
A sample
$\samp=\paren{X_i,Y_i}_{i=1}^n$ 
is drawn iid from an unknown distribution 
$\D$ 
over $\X\times\Y$.
For $x \in \X$ let $(X_{\pi_1(x)},Y_{\pi_1(x)})$ be its nearest neighbor in $S$:
\[
\pi_1(x) = \argmin_{i\in[n]} \dist(X_i,x).
\]
Thus, the nearest-neighbor classifier $\gnn$ is given by
\beqn
\label{eq:hnn-def}
\gnn(x) = Y_{\pi_1(x)}.
\eeqn
Define the function $\veta:\X\to\R^\Y$ by $$\veta(x)=\P(Y=\cdot\gn X=x).$$
The {\em Bayes optimal} classifier $g^*$ --- i.e., one that
minimizes $\P(g(X)\neq Y)$ over all measurable $g\in\Y^\X$ 
---
is well-known to have the form
\beq
g^*(x) = \argmax_{y\in\Y}\eta_y(x),
\eeq
where ties are broken arbitrarily.
Our only distributional assumption is that $\veta$ is $L$-Lipschitz with respect to the sup-norm.
Namely, for all $x,x'\in\X$, we have
\beq
\nrm{\veta(x)-\veta(x')}_\infty 
\equiv \sup_{y\in\Y}\abs{\eta_y(x)-\eta_y(x')}
\le L\dist(x,x').
\eeq
This is a direct analogue of the Lipschitz assumption for the binary case \citep{cover1967nearest,shai2014}.
We make the additional standard assumption that $\X$ has a finite doubling dimension: $\ddim(\X)=D<\infty$.
The Lipschitz and doubling assumptions are sufficient to extend
the finite-sample analysis of binary nearest neighbors \citep{shai2014}
to the multiclass case:
\begin{theorem} 
\label{thm:bayes}
\beq
\E_\samp\sqprn{\P(\gnn(X)\neq Y)}
\le
2\P(g^*(X)\neq Y)
+
\frac{4L}{ n^{{1}/{(D+1)}}}.
\eeq
\end{theorem} 
Note that the bound is independent of the number of classes $k$ and holds
even for $k=\infty$. The proof is deferred to Appendix~\ref{app:bayes_proof}.

\subsection{Multiclass margin bounds}
\label{sec:margin}
Here again $(\X,\dist)$ is a metric space,
but now the
label set $\Y$ is assumed finite: $|\Y|=k<\infty$. 
As before,
$\samp=\paren{X_i,Y_i}_{i=1}^n$ 
with $(X_i,Y_i)\sim\D$ iid.
It will be convenient to write 
$\samp^y = \{ X_i : Y_i = y, i\in[n]\}$ 
for the subset of examples with label $y$.
The metric induces the natural score function
\(
\fnn(x,y) =  - \dist(x,\samp^{y})
\)
with corresponding 
nearest-neighbor classifier
\beqn
\label{eq:gnn}
\gnn(x) = \argmax_{y\in \Y} \fnn(x,y),
\eeqn
easily seen to be identical to the one in (\ref{eq:hnn-def}).
At this point we make the simple but crucial observation that
the function $\fnn(\cdot,y): \X\to\R$ is $1$-Lipschitz.
This will enable us to generalize the powerful Lipschitz extension framework of
\citet{DBLP:journals/jmlr/LuxburgB04} to 
$\abs{\Y}>2$.

We will need a few definitions.
Let $\lipscore_L$ be the collection of all $L$-Lipschitz functions from $\X$ to $\R$
and put $\F_L=\lipscore_L\times\Y$.
Since each $f\in\F_L$ maps $\X\times\Y$ to $\R$, the margin
$\marg_f(x,y)$ is well-defined via (\ref{eq:margin_fun}).
Putting 
\beq
y_f^*(x) &=& \argmax_{y\in\Y} f(x,y),\\
\marg_f^*(x) &=& \marg_f(x,y_f^*(x))
,
\eeq
we define the {\em projection}
$\Phi_f$:
\beq
\Phi_f(x,y) &=& \begin{cases} 
\phantom{-}
 \marg_f^*(x), &\mbox{if }  y = y_f^*(x)
\\
-\marg_f^*(x), & \mbox{otherwise}. 
\end{cases} 
\eeq
Finally, we define $\TPhF_L$
as the truncated (as in (\ref{eq:T})) projections of functions in $\F_L$:
\beqn
\label{eq:mainClass}
\TPhF_L = 
\set{
(x,y) \mapsto \trun{\mhyphen1}{1}\paren{\Phi_f(x,y)} : f\in\F_L
}
.
\eeqn
Thus, $\TPhF_L$ is the set of functions $h_f:\X\times\Y\to[-1,1]$,
where each $h_f(\cdot,y)$ is $L$-Lipschitz and 
$h_f(x,y)=\pm\trun{\mhyphen1}{1}(\marg_f^*(x))$, 
depending upon whether $y=y_f^*(x)
$
, see Figure \ref{fig:mapping_diag} (left).

\subsubsection{Bounding the metric entropy}
\label{sec:metric-entropy}
\begin{figure}
\centering{
\begin{tikzpicture}[scale=1,every node/.style={scale=1, minimum size = 1mm, inner sep = 1mm}, every edge/.style={draw, line cap =round,  thin}
]

{
\node (u)   at (0,0) {$1$}
edge [-] (-.5,0)
edge [-|] (0,-1)
edge [-|] (0,1)
edge [-, ultra thick] (0,-.7); 
\node (r)  at (1,0) {$2$} 
 edge [-|] (1,-1) 
 edge [-|] (1,1)
 edge [-, ultra thick] (1,.7)
 edge [-] (u); 
 \node (dl)  at (2,0) {$3$}
 edge [-] (r)
 edge [-|] (2,-1)
 edge [-|] (2,1)
 edge [-] (2.5,0)
 edge [-, ultra thick] (2,-.7); 
 \node at (1,1.3) {$y^*_f$};
 \path 
 (0,-.7) node (f1)  [circle, fill, inner sep = 0pt] {}
  (1,.7) node (f2)  [circle, fill, inner sep = 0pt] {}
  node [right = .5mm of f2] {$\marg^*_f$}
 (2,-.7) node (f3)  [circle, fill, inner sep = 0pt] {}
 node [right = .5mm of f3] {$\mhyphen\marg^*_f$}
 node [left = 1.5mm of f1] {$\mhyphen\marg^*_f$}
 ;
 
\coordinate (c) at (5,0) ;
\node (u) at (5,1.1) [inner sep=1mm] {$1$}
edge [|-] (5,0);
\node (dr) at (6.1,-1) [inner sep=1mm] {$2$}
 edge [|-] (5,0);
 \node (dl) at (3.9,-1) [inner sep=1mm] {$3$}
 edge [|-] (5,0);
 \path 
 (5.66,-.6) node (p2) [circle, fill, inner sep = 0pt] {}
 edge [ ultra thick, line cap =round] (c)
  node [ above right = 1.5mm and 1mm of p2, scale = 1, inner sep = 0pt] {$(y^*_f, \marg^*_f)$}
 ;
 \node [below = 2.02cm of u] {$\wt{h}(x)$};
 \node [below = 1cm of r] {$h_f(x,y)$};
 \node at (3.5, 0) {\Large $\Rightarrow$};
 }
\end{tikzpicture}
}
\caption{The mapping in Lemma \ref{lem:covNum}
with $\abs{\Y}=3$.}
\label{fig:mapping_diag}
\end{figure}
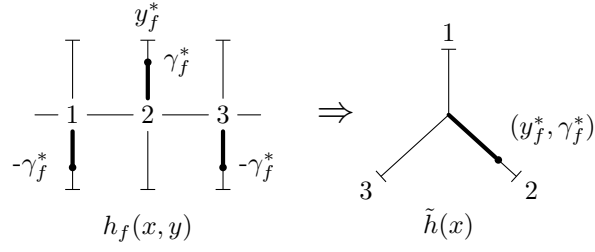

\begin{figure}
\centering{
\begin{tikzpicture}[scale=1,every node/.style={scale=1, minimum size = 0mm, inner sep =0pt, outer sep = 0pt}, every edge/.style={draw, line cap =round}]
{\small
\node (c) at (0,0) {};
\node (u) [inner sep = 1mm] at (0,1.1) {$1$}
edge [|-] (0,0);
\node (dr) [inner sep = 1mm] at (1.1,-1) {$2$}
 edge [|-] (0,0);
 \node (dl) [inner sep = 1mm] at (-1.1,-1) {$3$}
 edge [|-] (0,0);
 \path  (.22,-.2) node (p1) [circle, fill,  inner sep = 1pt]{}
 (.66,-.6) node (p2) [circle, fill, inner sep = 1pt] {}
 edge [ultra thick, line cap =round] (p1)
 (p2) node [ label={[ label distance=1mm]5:{\small $\wt{h}(x)$}}] {}
 (p1) node [inner sep = 0pt, label={[ label distance=1mm]85:$\wt{h}'(x)$}] {}
 ;
 \node [below = 1.8cm of u] {$y = y'$};
 
\coordinate (c) at (4,0) ;
\node (u) [inner sep = 1mm] at (4,1.1) {$1$}
edge [|-] (4,0);
\node (dr) [inner sep = 1mm] at (5.1,-1) {$2$}
 edge [|-] (4,0);
 \node (dl) [inner sep = 1mm] at (2.9,-1) {$3$}
 edge [|-] (4,0);
 \path  (4,.4) node (p1) [circle, fill,inner sep = 1pt]{}
 (4.66,-.6) node (p2) [circle, fill, inner sep = 1pt,outer sep = 0pt] {}
 edge [ultra thick, line cap =round] (c)
 (c) edge [ultra thick, line cap =round] (p1)
  (p2) node [inner sep = 0pt,outer sep = 0pt,label={[ label distance=1mm]20:$\wt{h}(x)$}] {}
 (p1) node [label={[ label distance=1.8mm]0:$\wt{h}'(x)$}] {}
 ;
 \node [below = 1.8cm of u] {$y \neq y'$};
 }
\end{tikzpicture}
}
\caption{The metric $\distf(\wt{h}(x),\wt{h}'(x))$
with $\abs{\Y}=3$.}
\label{fig:metric}
\end{figure}

Our main technical result is a bound on the metric entropy of
$\TPhF_L$, which will be used to obtain error bounds (Theorems \ref{thm:radem} and \ref{thm:scale}) 
for classifiers derived from this function class.
The analysis differs from previous bounds (see Table \ref{table:comp_bounds}) by explicitly taking advantage of the mutual exclusive nature of the labels, obtaining an exponential improvement in terms of the number of classes $k$.
Endow $\TPhF_L$ with the sup-norm
\[
\norm{\cdot}_\infty = \sup_{x\in\X} \max_{y\in\Y} \abs{\enskip\cdot\enskip}.
\]
\begin{lemma}
\label{lem:covNum}
For any $\eps>0$,
\beq
\log \mc{N}(\eps,\TPhF_L, \norm{\cdot}_\infty)
\leq
{\lp
\frac{16L}{\eps}
\rp^D
}
\log
\lp
\frac{5k}{\eps}
\rp.
\eeq
\end{lemma}
\begin{proof}
By the definition of $\TPhF_L$, for all $h_f \in \TPhF_L$ and $x\in \X$ 
there is at most one $y\in\Y$ such that $h_f(x,y)>0$.
In addition, if 
$h_f(x,y) = c>0$,
then
$h_f(x,{y'}) = -c$ for all $y'\neq y$. 
Since $\marg_f^*(x)\ge0$,
we may reparametrize $h_f(x,y)$ by
$(y_f^*(x),\marg_f^*(x)) \in \Y \times [0,1]$,
see Figure \ref{fig:mapping_diag}.
To complete the mapping
$h_f \mapsto (y_f^*,\marg_f^*)$,
define the following star-like metric $\distf$ over $\Y \times [0,1]$ (see Figure \ref{fig:metric}):
\[
\distf( (y,\marg) , (y', \marg') ) = 
\left\{
\begin{array}{cl}
\abs{\marg-\marg'} & y = y' \\
{\marg} + {\marg'} & y \neq y'
\end{array}
\right..
\]
Let
$\wt{H}_L$ 
be the collection
of functions $\wt{h}: \X \to \Y \times [0,1]$ 
that are $L$-Lipschitz:
\[
\distf( \wt{h}(x), \wt{h}(x') )\leq L \dist(x,x'), \qquad x,x'\in\X.
\]
It is easily verified that the metric space $(\TPhF_L, \norm{\cdot}_\infty)$ is {isometric} to 
$(\wt{H}_L, \distf_\infty)$ with 
\[
\distf_\infty(\wt{h},\wt{h}') =  \sup_{x \in \X} \distf( \wt{h}(x),\wt{h}'(x) ).
\]
Thus, $\mc{N}(\eps,{\TPhF}_L, \norm{\cdot}_\infty)
=
\mc{N}(\eps,\wt{H}_L, \distf_\infty)$, and we proceed to bound the latter.\footnote{The
remainder of the proof is based on a technique
communicated to us by R. Krauthgamer,
a variant of the classic \citet{kolmogorov1959varepsilon} method.
}
Fix a covering of $\X$ consisting of $|N| = \mc{N} (\e/8L, \X , \dist)$ balls $\{U_1,\dots,U_{|N|}\}$ of radius $\e' = \e/8L$ 
and choose $|N|$ points $N = \{x_i \in U_i\}_{i=1}^{|N|}$.
Construct
$\wh{H}\subset\wt{H}_{2L}$
as follows.
At every point $x_i\in N$ select one of the classes $y\in\Y$ and set $\wh{h}(x_i) = (y, \marg(x_i) )$ 
with $\marg(x_i)$ some multiple of $2L\e' = \e/4$,
while maintaining
$\Lipts{\wh{h}}\le 2L$. 
Construct a $2L$-Lipschitz  extension for $\wh{h}$ from $N$ to all over $\X$ (such an extension always exists, 
\citep{MR1562984,Whitney1934}).
We claim that every classifier in $\TPhF_L$, via its twin $\wt{h} \in \wt{H}_L$, is close to some 
$\wh{h} \in \wh{H}$, in the sense that $\distf_\infty(\wt{h}, \wh{h}) \leq \e$. 
Indeed,
every point $x \in \X$ is $2\e'$-close to some point $x_i \in N$, and since 
${\wt{h}}$ is $L$-Lipschitz and $\wh{h}$ is $2L$-Lipschitz,
\beq
{\distf}(\wt{h}(x), \wh{h}(x))
 & \leq  &
\distf(\wt{h}(x), \wt{h}(x_i))\\
& & + \: \distf(\wt{h}(x_i), \wh{h}(x_i)) \\
& & + \: \distf(\wh{h}(x_i), \wh{h}(x))\\
& \leq &
L \dist(x, x_i) + \e/4 + 2L\dist(x, x_i) \\
& \leq & \e.
\eeq
Thus, $\wh{H}$ provides an $\e$-cover for $\wt{H}_L$ (and hence for $\TPhF_L$). Note that
\(
|\wh H| \leq  (\ceil{ 4k/\e} + 1)^{|N|}
\)
, since by construction,
functions $\wh{h}$ are determined by
their values on $N$, which at a given point can take one of $\ceil{ 4k/\e} + 1$ possible values.
Since by (\ref{eq:X_cov_Num}) we have
\(
|N| =
\mc{N}({\eps}/{8L},\X, \dist)
\leq
\lp
\frac{16L}{\eps}
\rp^D
\)
the bound follows.
\end{proof}
A tighter bound is possible when the metric space $(\X,\dist)$ possesses two additional properties:
\begin{enumerate}
\item
$(\X,\dist)$ is {\em connected} if for all $x,x'\in\X$ and all $\eps>0$,
there is a finite sequence of 
points $x=x_1,x_2,\ldots,x_m=x'$
such that 
$\dist(x_i,x_{i+1})<\eps$
for all $1\le i<m$.
\item
$(\X,\dist)$ is {\em centered} if 
for all $r>0$ and
all $A \subset \X$ with
$\diam(A) \leq 2r$, there exists a point $x \in \X$ such that $\dist(x, a) \leq r$ for all $a \in A$.
\end{enumerate}

\begin{lemma} If $(\X,\dist)$ is connected and centered, then
\label{lem:tight_CovNum}
\beq
\log \mc{N}(\eps,\TPhF_L, \norm{\cdot}_\infty)
= O\lp
{\lp
\frac{L}{\eps}
\rp^D
}
\log k
+ \log \lp \frac{1}{\e}\rp
\rp
.
\eeq
\end{lemma}
\begin{proof} 
With the additional assumptions on $\X$ we
follow the proof idea in \citet{kolmogorov1959varepsilon}
and demonstrate the tighter bound
\(
|\wh H| \leq (\left\lceil 4k/\e \right\rceil + 1) (2k+1)^{|N|-1} =O ( (2k)^{|N|}/\e ).
\)
Here $\wh{H}$ is constructed as in the proof for Lemma~\ref{lem:covNum} but now each $x_i\in N$ is taken to be a ``center'' of $U_i$,
as furnished by Property 2 above.
Let $x_j\in N$.
Since $\X$ is connected, we may traverse a path from 
$x_1$ to $x_j$ via the cover points
$x_1 = x_{i_1},x_{i_2},\dots, x_{i_m} = x_j$, 
such that the distance between any two successive points $(x_{i_l},x_{i_{l+1}})$ is at most 
$2\e' = \e/4L$. 
Since $\wh{h}$ is $2L$-Lipschitz, on any two such points 
the value of $\wh{h}$ can change by at most $\e/2$. 
Thus, given the value $\wh{h}(x_{i_l})$, the value of $\wh{h}(x_{i_{l+1}})$ can take one of at most 
$2k+1$ values 
(as Figure \ref{fig:metric} shows, at the star's hub, 
$\wh{h}(x_{i_{l+1}})$ can take one of $2k+1$ values, 
while 
at one of the spokes only $5$ values are possible).
So we are left to choose the value of $\wh{h}$ on the point $x_1$ to be one from the $\left\lceil 4k/\e \right\rceil + 1$ possible values. The bounds on $|\wh{H}|$ and the metric entropy follow. 
\end{proof}

\subsubsection{Rademacher analysis}%
\label{sec:rademacher}

The {\em Rademacher complexity} of 
the set of functions 
$\TPhF_L$
is defined by
\beqn
\label{eq:rad}
{\radem}_n(\TPhF_L) = 
\E
\left[ 
\sup_{h\in\TPhF_L} 
\oo{n}\sum_{i=1}^n \sigma_i h(X_i,Y_i)
\right],
\eeqn
where the $\sigma_i$ are $n$ independent 
random variables 
with $\P(\sigma_i=+1)=\P(\sigma_i=-1)=1/2$.
In 
Appendix~\ref{app:radem_proof},
we 
invoke
Lemma~\ref{lem:covNum}
to
derive the bound
\beqn
\label{eq:radem}
{\radem}_n(\TPhF_L) 
\leq
2 L \lp \frac{\log 5k}{n} \rp^{1/(D+1)},
\eeqn
which in turn implies
``half'' of our main risk estimate 
(\ref{eq:intro_bound}):
\begin{theorem}
\label{thm:radem}
With probability at least $1-\delta$, for all $L>0$
and every $f\in\F_L$
with its 
projected version
$h_f\in \TPhF_L$,
\beq
\P(g_f(X) \neq Y )  \leq  
\wh\E\sqprn{ \loss({h_f})}
  + \Delta\Rad (n,L,\delta),
\eeq
where 
$g_f$ is the classifier defined in (\ref{eq:g_f}),
$\loss$ is any of the loss functions 
defined in Section~\ref{sec:prelim}
and
$\Delta\Rad (n,L,\delta)$ is at most
\beq
\quad 8 L \lp \frac{\log 5k}{n} \rp^{\oo{D+1}}
+ \sqrt{\pl{\frac{\log\log_2 {2L}}{n}}}
+ \sqrt{\frac{\log{\frac{2}{\delta}}}{2n}}
.
\eeq
\end{theorem}
\subsubsection{Scale-sensitive analysis}
\label{sec:scale-sensitive}
The following Theorem
, proved in Appendix \ref{app:scale_proof}, 
is an adaptation of \citet[Theorem 1]{MR2383568}, using Lemma~\ref{lem:covNum}.
\begin{theorem}
\label{thm:scale}
With probability at least $1-\delta$, for all $L>0$
and every $f\in\F_L$
with its induced 
$h_f\in \TPhF_L$,
\beq
\P( g_f(X) \neq Y )  \leq 
\wh\E\sqprn{\loss\cutoff({h_f})} 
+ \Delta_{\fat}(n,L,\delta),
\eeq
where 
$\Delta_{\fat}(n,L,\delta)$ is at most
\begin{align*}
 \sqrt{
 \frac{2 }{n}
 \lp
2
\lp
{16L}
\rp^{D} 
\log
 \lp
20k
 \rp
 + \ln \lp \frac{2 L}{ \delta} \rp
 \rp
 }
 +\oo{n}.
 \end{align*}
\end{theorem}

\subsubsection{Combined Bound}
\label{sec:main-theorem}
Taking $\loss = \loss\cutoff$ in Theorem \ref{thm:radem} we can merge the above two bounds by taking their minimum.
Namely,
Theorem~\ref{thm:scale} 
holds
with
\(
\Delta(n,L,\delta)  = \min\set{\Delta\Rad(n,L,\delta),\Delta_\fat(n,L,\delta)}
\)
in place of $\Delta_{\fat}(n,L,\delta)$,
see Figure~\ref{fig:bounds}. 
The resulting risk decay rate is 
of order
\beq
    \min
        \left\{
L \lp \frac{\log k}{n} \rp ^ {\frac{1}{D + 1}},
L^{\frac{D}{2}} \lp{\frac{\log k}{n}}\rp^{\oo{2}}
      \right\}
,
\eeq
as claimed in (\ref{eq:intro_bound}).
In terms of the number of classes $k$,
our bound compares favorably to those
in \citet{allwein2001reducing, MR2383568, MR2745296}, and more recently in \citet{DBLP:journals/jmlr/DanielySBS11}, 
which
have a $k$-dependence of $O(d\natr\log k)$, 
where $d\natr$ is the (scale-sensitive, $k$-dependent) 
Natarajan dimension of the multiclass hypothesis class. 
The optimal dependence of the risk on $k$ is an intriguing open problem.
\begin{figure}%
\centering
\includegraphics[scale=.45]{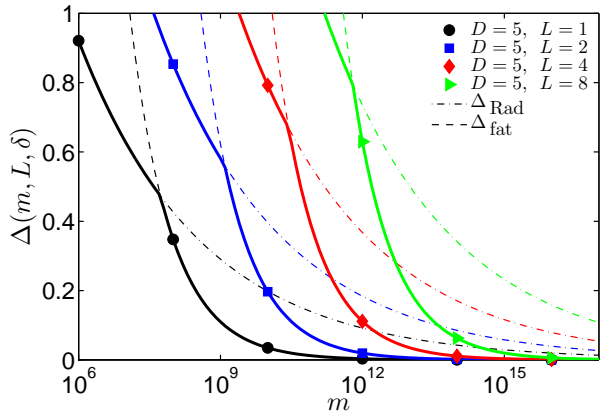}%
\caption{The combined complexity bounds ($k=10, \delta=0.01$).}%
\label{fig:bounds}%
\end{figure}

\section{Algorithm}
\label{sec:alg}
Theorems~\ref{thm:radem} and \ref{thm:scale} yield generalization bounds of the schematic form
\beqn
\label{eq:alg_bound}
\P(g(X)\neq Y) \le \wh\E[\loss] + \Delta(n,L,\delta).
\eeqn
The free parameter $L$ in (\ref{eq:alg_bound}) controls
(roughly speaking) the bias-variance tradeoff:
for larger $L$, we may achieve a smaller empirical loss
$\wh\E[\loss]$ at the expense of a larger hypothesis complexity 
$\Delta(n,L,\delta)$.
Our Structural Risk Minimization 
(SRM) consists of seeking the optimal $L$ --- i.e.,
one that minimizes the right-hand side of (\ref{eq:alg_bound}) ---
via the following high-level procedure:
\begin{enumerate}
\item For each $L>0$, minimize $\wh\E[\loss(h_f)]$ over $f\in\F_L$.
\item Choose the optimal $L^*$ and its corresponding classifier $g_f$ with
$f\in\F_{L^*}$.
\end{enumerate}

\paragraph{Minimizing the empirical loss.}
Let $\samp=\paren{X_i,Y_i}_{i=1}^n$
be the training sample
and
$L>0$
a given maximal allowed Lipschiz constant.
We will say that 
a function $h \in \TPhF_L$ is {\em inconsistent}
with a sample point $(x,y)$ if $h(x,y)<1$
(i.e., if the margin of $h$ on $(x,y)$ is less than one).
Denote by
$\wh{\mist}(L)$ 
the smallest possible number of 
sample points
on which a function $h\in\TPhF_L$
may be inconsistent:
\[
\wh{\mist}(L) = \min_{h \in \TPhF_L}
\wh{\E}[{\loss\cutoff(h)}].
\]
Thus, our SRM problem consists of finding
\[
L^* = 
\argmin_{L>0}
\left\{
\wh{\mist}(L) + \Delta(n,L,\delta)
\right\}.
\]
For $k=2$,
\citet{DBLP:conf/colt/GottliebKK10}
reduced 
the problem of computing $\wh{\mist}(L)$
to one of finding a minimal vertex cover in a bipartite graph
(by K\"onig's theorem, the latter is efficiently computable
as a maximal matching).
We will extend this technique to $k>2$ as follows.
Define the
$k$-partite graph $G_L = (\{V^y\}_{y=1}^k, E)$,
where
each vertex set $V^y$ corresponds to the sample points $S^y$
with label $y$.
Now in order for 
$h \in \TPhF_L$
to be consistent with
the points 
$(X_i,Y_i)$ and $(X_j,Y_j)$ 
for $Y_i\neq Y_j$, the following relation must hold:
\beqn
\label{eq:edge-cond}
L\dist(X_i,X_j)\ge2.
\eeqn
Hence, we define the edges of $G_L$ to consist of all point
pairs violating (\ref{eq:edge-cond}):
\beq
(X_i,X_j)\in E &\iff& (Y_i\neq Y_j ) \wedge (\dist(X_i,X_j)<2/L).
\eeq
Since removing either of $X_i,X_j$ in (\ref{eq:edge-cond}) also deletes the 
violating edge, $\wh{\mist}(L)$ is 
by construction equivalent to the size of the minimum vertex cover for $G_L$.
Although minimum vertex cover is NP-hard to compute (and even hard to approximate
within a factor of 1.3606, \citep{dinur2005}),
a $2$-approximation may be found in $O(n^2)$ time
\citep{papadimitriou1982combinatorial}.
This yields a $2$-approximation 
$\wt{\mist}(L)$
for $\wh{\mist}(L)$.

\paragraph{Optimizing over $L$.}
Equipped with an efficient routine for computing $\wt{\mist}(L)\le2\wh{\mist}(L)$,
we now seek an $L>0$ that minimizes 
\beqn
\label{eq:wt}
{Q}(L) := \wt{\mist}(L) + \Delta(n,L,\delta).
\eeqn
Since the Lipschitz constant induced by the data is determined by
the $\binom{n}{2}$ distances among the sample points,
we need only consider $O(n^2)$ values of $L$.
Rather a brute-force searching all of these values,
Theorem 7 of \citet{DBLP:conf/colt/GottliebKK10} shows 
that 
using an $O(\log n)$ time binary search over the values of $L$,
one may 
approximately minimize $Q(L)$, which in turn
yields an approximate solution to (\ref{eq:alg_bound}).
The resulting procedure has runtime
$O(n^2\log n)$ and guarantees
an $\wt{L}$ for which
\beqn
\label{eq:Q}
Q(\wt{L})
&\le&
4\sqprn{\wh{\mist}({L}^*) + \Delta(n,{L}^*,\delta)}.
\eeqn

\paragraph{Classifying test points.}
Given the nearly optimal Lipschitz constant $\wt{L}$ computed above
we 
construct the approximate (within a factor of 4)
empirical risk minimizer $h^*\in\TPhF_{\wt{L}}$.
The latter partitions the sample into $S=S_0\cup S_1$,
where $S_1$ consists of the points on which $h^*$
is consistent
and $S_0=S\setminus S_1$.
Evaluating $h^*$ on a test point
amounts to
finding its nearest neighbor in $S_1$.
Although in general metric spaces, nearest-neighbors search requires
$\Omega(n)$ time, for doubling spaces, an exponential speedup is available
via approximate nearest neighbors (see Section \ref{sec:ext}).

\section{Extensions}
\label{sec:ext}
In this section, we discuss two approaches
that render
the methods presented above considerably 
more efficient in terms of runtime and
generalization bounds.
The first is based on the fact that in
doubling spaces,
hypothesis evaluation time may be reduced
from $O(n)$ to $O(\log n)$ at the expense of a very slight degradation
of the generalization bounds.
The second relies on a recent metric dimensionality reduction result.
When the data is ``close'' to being $\tilde D$-dimensional,
with $\tilde D$ much smaller than the ambient metric space dimension $D$,
both the evaluation runtime and the generalization bounds
may be significantly improved --- depending essentially on $\tilde D$
rather than $D$.

\subsection{Exponential speedup via approximate NN}
\label{sec:fast}
If $(\X,d)$ is a metric space and
$x^*\in E\subset\X$ is a minimizer of $\dist(x,x')$ over $x'\in E$, then $x^*$ is
a {\em nearest neighbor} of $x$ in $E$. 
A simple information-theoretic argument shows that the time complexity of computing
an exact nearest neighbor in general metric spaces has $\Omega(n)$ time complexity.
However, an exponential speedup is possible if 
(i) $\X$ is a doubling space and (ii) one is willing
to settle for approximate nearest neighbors.
A $(1+\eta)$ nearest neighbor oracle returns an
$\wt{x}\in E$ such that 
\beqn
\label{eq:app_metric}
\dist(x,x^*)\le\dist(x,\wt{x})\le(1+\eta)\dist(x,x^*).
\eeqn

We will use the fact that 
in a doubling space,
one may precompute a $(1+\eta)$ nearest neighbor data structure in 
$(2^{O(\ddim(\X))}\log n + \eta^{-O(\ddim(\X))})n$  
time and evaluate it on a test point in 
$2^{O(\ddim(\X))}\log n + \eta^{-O(\ddim(\X))}$  
time \cite{cole2006searching,har2006fast}.
The approximate nearest neighbor oracle induces
an $\eta$-approximate version of $\gnn$ in defined (\ref{eq:gnn}).
After performing SRM as described in Section~\ref{sec:alg},
we are left with a subset $S_1\subset S$ of the sample, which
will be used to label test points.
More precisely, the predicted label of a test point will be determined by
its $\eta$-nearest neighbor in $S_1$.

The exponential speedup afforded by approximate nearest neighbors
comes at the expense of mildly degraded generalization guarantees.
The modified generalization bounds
are derived in
three steps,
whose details are deferred to Appendix~\ref{app:app}:\\
{(i)}
We cast the evaluation of
$h\in \TPhF_L$ 
in (\ref{eq:mainClass})
as a nearest neighbor calculation 
with a corresponding $\wt{h}$ induced by the $(1+\eta)$ approximate
nearest neighbor oracle.
The nearest-neighbor formulation of $h$ is essentially the one obtained by
\citet{DBLP:journals/jmlr/LuxburgB04}:
\begin{alignat}{2}
\label{eq:lip_ext}
h(x,y) = 
\oo{2}
&  \lp
    \min_{
\samp_1} 
         \right.            
         \left\{
\xi(y,y')
+ L\dist(x,x')
    \right\}
    \\
    \nonumber
& \quad \left.   +\,
    \max_{
\samp_1}    
    \left\{
\xi(y,y')
- L\dist(x,x')
    \right\}
\rp,
\end{alignat}
where
$(x',y')
\in
S_1$ and
$\xi(y,y')=2\pred{y=y'}-1$.\\
{(ii)}
We observe a simple relation between $h$ and $\wt{h}$:
\beq
\tsnrm{h-\wt{h}}_\infty \equiv
\sup_{x\in\X,y\in\Y}\tsabs{h(x,y) - \wt{h}(x,y)}\le 2\eta.
\eeq
{(iii)}
Defining the $2\eta$-perturbed function class
\beq
\TPhF_{L,2\eta} 
= \{\trun{\mhyphen1}{1} (h'): \nrm{h'-h}_\infty \leq 2\eta , 
h \in \TPhF_{L} \},
\eeq
we relate its metric entropy to that of 
$\TPhF_{L}$:
\begin{lemma} For $\eps > 2\eta > 0$, we have
\label{lem:app_cov_num}
\beq
\mc{N}(\eps,\TPhF_{L,{2\eta}}, \norm{\cdot}_\infty) 
& \leq & \mc{N}({\eps-2\eta},\TPhF_{L}, \norm{\cdot}_\infty)
.
\eeq
\end{lemma}
The metric entropy estimate for
$\TPhF_{L,{2\eta}}$ readily yields
$\eta$-perturbed versions of Theorems 
\ref{thm:radem} and \ref{thm:scale}.
From the standpoint of generalization bounds,
the effect of the $\eta$-perturbation
on $\TPhF_{L}$ amounts, roughly speaking,
to replacing $L$ by $L(1+O(\eta))$,
which constitutes a rather benign degradation.

\subsection{Adaptive dimensionality reduction}
\label{sec:adapt}
The generalization bound in (\ref{eq:intro_bound})
and the runtime of our sped-up algorithm in Section~\ref{sec:fast}
both depend exponentially on the doubling dimension of the metric space.
Hence, even a modest dimensionality reduction could lead to dramatic savings
in algorithmic and sample complexities. The standard Euclidean dimensionality-reduction
tool, PCA, until recently had no metric analogue --- at least not with rigorous performance
guarantees. The technique proposed in \citet{gkr2013-alt} may roughly be described 
as a metric analogue of PCA.

A set $X=\set{x_1,\ldots,x_n}\subset\X$ inherits the metric $\dist$ of $\X$ and hence 
$\ddim(X)\le\ddim(\X)$ is well-defined.
We say that $\tilde X=\set{\tilde x_1,\ldots, \tilde x_n}\subset\X$ is an
$(\alpha,\beta)$-{\em perturbation} of $X$ if
$\sum_{i=1}^n\dist(x_i,\tilde x_i)\le\alpha$ and $\ddim(\tilde X)\le\beta$.
Intuitively, the data is ``essentially'' low-dimensional if it admits
an $(\alpha,\beta)$-{perturbation} with small $\alpha,\beta$,
which leads to improved Rademacher estimates.
The {\em empirical Rademacher complexity} of $\TPhF_L$ 
on a sample $S=(X,Y)\in\X^n\times\Y^n$
is given by
\beq
\wh{\radem}_n(\TPhF_L;S) = 
\E
\sqprn{
\left.
\sup_{h\in\TPhF_L} 
\oo{n}\sum_{i=1}^n \sigma_i h(X_i,Y_i)
\,\right\vert
S\,
}
\eeq
and is related to $\radem_n$ defined in (\ref{eq:rad})
via 
\beq
\radem_n(\TPhF_L) &=& \E_S\sqprn{\wh{\radem}_n(\TPhF_L;S)}
\\
\P\paren{\abs{ \radem_n - \wh{\radem}_n } \ge \eps} &\le& 2\exp(-\eps^2n/2),
\eeq
where the identity is obvious and the inequality is a simple consequence of measure concentration 
\citep{mohri-book2012}. Hence, up to small changes in constants, the two may be used in generalization
bounds such as Theorem~\ref{thm:radem} interchangeably.
The data-dependent nature of $\wh{\radem}_n$ lets us exploit essentially low-dimensional data
(see Appendix~\ref{app:dim-red-proof}):
\begin{theorem}
\label{thm:rad-dim-red}
Let $S=(X,Y)\in\X^n\times\Y^n$ be the training sample
and suppose that $X$ admits
an
$(\alpha,\beta)$-perturbation $\tilde X$.
Then
\beqn
\label{eq:rad-dim-red}
\wh{\radem}_n(\TPhF_L;S) = O\paren{ L\paren{\alpha +\paren{\frac{\log k}{n}}^{\oo{1+\beta}}}}.
\eeqn
\end{theorem}
A pleasant feature of the bound above is that it does not depend on $\ddim(\X)$
(the dimension of the ambient space)
or even on $\ddim(X)$
(the dimension of the data).
Note the inherent tradeoff between the distortion $\alpha$
and dimension $\beta$, with some non-trivial $(\alpha^*,\beta^*)$ minimizing the right-hand side
of (\ref{eq:rad-dim-red}). Although computing the optimal 
$(\alpha^*,\beta^*)$ seems computationally difficult, 
\citet{gkr2013-alt} 
were able to obtain an 
efficient
$(O(1),O(1))$-{\em bicriteria approximation}.
Namely, their algorithm computes an 
$\tilde\alpha\le c_0\alpha^*$
and
$\tilde\beta\le c_1\beta^*$,
with the corresponding perturbed set $\tilde X$,
for universal constants $c_0,c_1$,
with a runtime of
$2^{O(\ddim(X))}n\log n+O(n\log^5n)$.

The optimization routine over $(\alpha,\beta)$ may then be embedded inside
our SRM optimization over the Lipschitz constant $L$
in Section~\ref{sec:alg}.
The end result will be a nearly optimal (in the sense of (\ref{eq:Q}))
Lipschitz constant $\wt{L}$, which induces the partition $S=S_0\cup S_1$,
as well as $(\tilde\alpha,\tilde\beta)$, which induce the perturbed set $\tilde S_1$.
To evaluate our hypothesis on a test point, we may invoke the $(1+\eta)$-approximate nearest-neighbor
routine from Section~\ref{sec:fast}.
This involves a precomputation of time complexity
$(2^{O(\tilde\beta)}\log n + \eta^{-O(\tilde\beta)})n$,
after which new points are classified in
$2^{O(\tilde\beta)}\log n + \eta^{-O(\tilde\beta)}$
time.
Note that the evaluation time complexity depends only on the ``intrinsic dimension''
$\tilde\beta$ of the data, rather than the ambient metric space dimension.

\appendix
\section{Bayes near-optimality proof}
\label{app:bayes_proof}
\begin{proof}[Proof of Theorem~\ref{thm:bayes}] Since $\veta$ is $L$-Lipschitz, given $x,x'\in\X$ we have
\begin{alignat}{2}
\label{eq:YneqY_bound}
 \D(Y \neq Y' \gn x,&x')
= 
\sum_{j\in\Y} \bet_j(x) ( 1 - \bet_j(x'))   \\
\nonumber
& \leq 
\sum_j \bet_j(x) \lp 1 - \bet_j(x) + L\dist(x,x') \rp\\
\nonumber
 &= 
\sum_j \bet_j(x) \lp 1 - \bet_j(x) \rp + L\dist(x,x').
\end{alignat}
By the definition of the nearest neighbor classifier $\gnn$ in (\ref{eq:hnn-def}) we have $\E_{\samp}[\D(\gnn(X) \neq Y)]  =  \E_{\samp} [\D( Y_{\pi_1(X)} \neq Y)]$, where the expectation is over the sample $\samp$ determining $\gnn$. 
By (\ref{eq:YneqY_bound}) this error is bounded above by
\beq 
\E_{\samp,X}[\sum_j \bet_j(X) ( 1 - \bet_j(X))]
  +  L\E_{\samp,X}[\dist(X,X_{\pi_1(X)})],
\eeq
where now the expectation is over $S$ and $X$.
Denoting $k'=\argmax_j \bet_j(X)$ and splitting the sum 
, the first term (which does not depend on $\samp$) satisfies
\begin{alignat*}{2}
 \E_{X}[\bet_{k'}(X) ( 1 - &\bet_{k'}(X))]
 + \E_{X}[\sum_{j\neq k'} \bet_j(X) ( 1 - \bet_j(X))] 
\\
& \leq 
 \E_{X}[1-\bet_{k'}(X)]  +  \E_{X}[\sum_{j\neq k'} \bet_j(X)] 
\\
& = 
 2 \E_{X}[1-\bet_{k'}(X)]
 = 
 2 \D(g^*(X) \neq Y).
\end{alignat*}
It remains to bound $\E_{\samp,X}[\dist(X,X_{\pi_1(X)})]$ and
we proceed exactly as in \citet{shai2014}.
Let $\{C_1,\dots,C_N\}$ be an $\eps$-cover of $\X$ of cardinality $N=\mc{N}(\eps,\X,\dist)$.
Given a sample $\samp$, for $x\in C_{i}$ such that $\samp \cap C_i \neq \emptyset$ 
we have $\dist(x,X_{\pi_1(x)})<\e$, while for $x \in C_{i}$ such that $\samp \cap C_i = \emptyset$ we have $\dist(x,X_{\pi_1(x)}) \leq \diam(\X) =1$,
thus $\E_{\samp,X}[\dist(X,X_{\pi_1(X)})]$ is bounded above by 
\beq
& \leq &\E_{\samp} 
    \left[
        \sum_{i=1}^N  \D(C_i) 
            \lp
                 \e\, \chr_{\{\samp \cap C_i \neq \emptyset\}} 
                 + \chr_{\{\samp \cap C_i = \emptyset\}} 
            \rp
    \right]\\
& = &
    \sum_{i=1}^N  \D(C_i) 
            \lp
                 \e\, \E_{\samp}\left[\chr_{\{\samp \cap C_i \neq \emptyset\}}\right] 
                 + \E_{\samp}\left[\chr_{\{\samp \cap C_i = \emptyset\}}\right] 
            \rp.
\eeq
Since $\D(C_i) \E_{\samp}[\chr_{\samp \cap C_i = \emptyset}] = \D(C_i) (1-\D(C_i))^n \leq 1/en$ and 
$N=\mc{N}(\eps,\X,\dist)$
we get
\beq
\E_{\samp,X}[\dist(X,X_{\pi_1(X)})]
&\leq& 
\eps + \frac{ \mc{N}(\eps,\X,\dist) }{en}\\
&\leq& 
\eps + \frac{1}{en} \paren{\frac{2}{\eps}}^{D}.
\eeq%
Setting $\eps = 2n^{-\frac{1}{D+1}}$ 
concludes the proof.
\end{proof}
\section{Rademacher analysis proofs}
\label{app:radem_proof}
\begin{proof}[Proof of inequality (\ref{eq:radem})]
Dudley's chaining integral \citep{Dudley1967290} bounds from above the Rademacher complexity ${\radem}_n(\TPhF_L)$ by
\beq
 \inf_{\a > 0}
\lp
4\a + 12\int_{\a}^\infty \sqrt{\frac{\log \mc{N}(t,\TPhF_L, \norm{\cdot}_\infty)}{n}} dt
\rp.
\eeq
By Lemma~\ref{lem:covNum} the integral can be bounded as follows:
\beq
& &\int_{\a}^\infty \sqrt{\frac{\log \mc{N}(t,\TPhF_L, \norm{\cdot}_\infty)}{n}} dt %
\\
&  \leq  &
\int_{\a}^\infty \sqrt{\frac{1}{n}\lp\frac{16L}{t}\rp^D \log \lp \frac{5k}{t} \rp}dt
\\
& \leq &
\int_{\a}^\infty \sqrt{\frac{\log 5k}{n}\lp\frac{16L}{t}\rp^D  \lp   \frac{1}{t} \rp}dt
\\
& = &
\sqrt{\frac{\log 5k}{n}}\lp {16L}\rp^{D/2} \int_{\a}^\infty  \lp \frac{1}{t} \rp^{(D+1)/2}dt
\\
& = &
\sqrt{\frac{\log 5k}{n}}\lp {16L}\rp^{D/2} \lp\frac{2}{D-1}\rp\lp\frac{1}{\a^{(D-1)/2}} \rp,
\eeq
where in the second inequality we used the fact that for $ x \in (0,1] $ and $c \geq e$ we have
\(
\log(\frac{c}{x}) \leq \frac{\log c}{x}.
\)
Choosing
\[
\a^*= \lp {9} (16L)^D  \frac{\log 5k}{n} \rp^{1/(D+1)}
\]
yields the bound.
\end{proof}
\begin{proof}[Proof of Theorem \ref{thm:radem}]
An adaptation\footnote{essentially setting $\a=1$ in \citet{mohri-book2012} and doing the stratification on $L$ instead} of \citet[Theorem 4.5]{mohri-book2012} to $\TPhF_L$ states 
that with probability $1-\delta$, for all $L>0$, $h\in \TPhF_L$,
\beq
\E[\loss\margin(h)]
& \leq & \wh{\E}[\loss\margin({h})]
 + 4\radem_n(\TPhF_L) \\
& + &   
\sqrt{\pl{\frac{\log\log_2 {2L}}{n} }}
+ \sqrt{\frac{\log{\frac{2}{\delta}}}{2n}}
.
\eeq
Since $\pred{u < 0} \leq \loss\margin(u)$ we have $P( g_h(X) \neq Y ) \leq \E[\loss\margin(h)]$. 
Since $\loss\margin(u) \leq \loss\cutoff(u)$ we can replace $\loss\margin$
in the empirical loss by the loss function $\loss\cutoff$.
Bounding
$\radem_n(\TPhF_L)$ using (\ref{eq:radem}) concludes the proof.
\end{proof}

\section{Scale sensitive analysis proof}
\label{app:scale_proof}
 \begin{proof}[Proof of Theorem~\ref{thm:scale}]
An application\footnote{setting $\gamma =1$ in \citet[Theorem 1]{MR2745296} and doing the 
stratification on $L$ instead} of \citet[Theorem 1]{MR2745296} states that 
with probability $1-\delta$, for all $L>0$, $h\in \TPhF_L$,
\beq
P( g_h(X) \neq Y )
 \leq  \oo{n}\sum_{i=1}^n \chr_{\{h(X_i,Y_i) < 1\}}
\hspace{2.5cm} \\
 + \quad \sqrt{
 \frac{2 }{n}
 \lp
2\log \mc{N}
\lp
1/4,\TPhF_L, \norm{\cdot}_\infty
\rp
 + \ln \lp \frac{2 L}{ \delta}
\rp
\rp}
+\oo{n}.
\eeq
Applying the metric entropy bound in Lemma \ref{lem:covNum} proves the Theorem.
\end{proof}

\section{Approximate NN proofs}
\label{app:app}
First, we will show that $\wt{h}$ is indeed a $2\eta$ additive perturbation of $h$, i.e. 
\beqn
\label{eq:app_h_sup}
\norm{h - \wt{h}}_{\infty} \leq 2\eta.
\eeqn
Instead of working directly with (\ref{eq:lip_ext}) we consider the following $L$-Lipschitz extension
\beq
h(x,y) &=& \oo{2}\trun{\mhyphen1}{1} 
 \lp
    \min_{\samp_1}       
         \left\{
            \xi(Y_i,y)  + L\dist(X_i,x)
    \right\}
\rp\\
&+& \oo{2}\trun{\mhyphen1}{1} 
 \lp
    \max_{\samp_1}       
         \left\{
            \xi(Y_i,y)  - L\dist(X_i,x)
    \right\}
\rp,
\eeq
easily seen to induce the same classifier $g_h$ as (\ref{eq:lip_ext}).
Consider the first term (the second term is treated similarly)
and its approximate version:
\beq
\wt{h}(x,y) &=& \trun{\mhyphen1}{1} 
 \lp
    \min_{S_1}       
         \left\{
            \xi(Y_i,y) + {L}\wt{\dist}(X_i,x)
    \right\}
\rp,
\eeq
where $\dist \leq \wt{d} \leq (1+\eta){d}$, given in (\ref{eq:app_metric}), is the approximate "`distance"' as provided by the approximate nearest neighbor.
For notational convenience, denote
\beq
h(x,y) &=& \trun{\mhyphen1}{1} (\min_{i} {q}_i (x,y))\\
\wt{h}(x,y) &=& \trun{\mhyphen1}{1} (\min_{i} \wt{q}_i (x,y))\\
{q}_i (x,y) &=& h_i(y) + r_i(x)\\
\wt{q}_i (x,y) &=& \wt{h}_i(y) + \wt{r}_i(x),
\eeq
where $h_i(y) 
= \xi(Y_i,y)$,
$r_i(x) = L\dist(X_i,x)$, 
and
$\wt{h}_i$, $\wt{r}_{i}$
defined analogously.

Observe that if $\wt{r}_i(x) > 2$ then ${r}_i(x) > 2/(1+\eta)\geq 2(1-\eta)$.
In this case, since $h$ has range in $[-1,1]$,
the eventual application of truncation operator $\tmoo$
will force $\wt{h}(x,y) - {h}(x,y)\leq 2\eta$.
Hence, we may assume that $\wt{r}_i(x) \leq 2$ and so ${r}_i(x) \leq 2$.
It is straightforward to verify that
for $a,b\in \R^n$ with $\max_{i\in[n]}\abs{a_i-b_i} \leq \eta$, we have
\beq
\abs{\tmoo(\min_{i} a_i)- \tmoo(\min_{i} b_i)} \leq \eta.
\eeq
Thus, establishing $\abs{q_i(x,y) - \wt{q}_i(x,y)} \leq 2\eta$ for all $i\in[|\samp_1|]$ and $y\in\Y$ 
with $\wt{r}_i(x), r_i(x) \leq 2$ suffices to prove the claim. Indeed,
by (\ref{eq:app_metric}) we have
\beq
\abs{r_i(x) - \wt{r}_i(x)} \leq \abs{r_i(x) - (1+\eta){r}_i(x)} \leq 2 \eta.
\eeq

\begin{proof}[Proof of Lemma \ref{lem:app_cov_num}.]
Suppose $\wt{h}\in \TPhF_{L,\eta}$. 
By the definition of $\TPhF_{L,\eta}$, there exists an 
$h\in\TPhF_L$ such that $\tsnrm{\wt{h} - h}_{\infty}\leq\eta$. 
Let ${h}'$ be some element in a minimal 
$\eps$-cover of $\TPhF_L$ so that $\tsnrm{{h} - h'}_{\infty}\leq\eps$. Then
\[
\tsnrm{\wt{h} - {h}'}_{\infty} \leq 
\tsnrm{\wt{h} - h}_{\infty} 
+ \tsnrm{h - {h}'}_{\infty} \leq 
\eps + \eta.
\]
Hence,
\[
\mc{N}(\eps + \eta,\TPhF_{L,\eta}, \norm{\cdot}_\infty) \leq 
\mc{N}(\eps ,\TPhF_L, \norm{\cdot}_\infty) ,
\]
whence the claim follows.
\end{proof}

\section{Dimensionality reduction proof}
\label{app:dim-red-proof}
\begin{proof}[Proof of Theorem~\ref{thm:rad-dim-red}]
Put $\tilde S=(\tilde X,Y)$.
For $X_i\in X$ and $\tilde X_i\in\tilde X$, define $\delta_i(h)=h(X_i,Y_i)-h(\tilde X_i,Y_i)$.
Then
\begin{alignat*}{2}
\wh{\radem}_n(\TPhF_L;S) & = 
\E\sqprn{\left.\sup_{h\in\TPhF_L} \oo{n}\sum_{i=1}^n \sigma_i h(X_i,Y_i)\,\right\vert S\,}
\\
& =
\E\sqprn{\left.\sup_{h\in\TPhF_L} \oo{n}\sum_{i=1}^n \sigma_i \paren{h(\tilde X_i,Y_i)-\delta_i(h)}
\,\right\vert S\,}\\
&\le
\wh{\radem}_n(\TPhF_L;\tilde S) 
+
\E\sqprn{\left.\sup_{h\in\TPhF_L} \oo{n}\sum_{i=1}^n \sigma_i \delta_i(h)
\,\right\vert S\,}.
\end{alignat*}
By (\ref{eq:radem}), we have
\beqn
\label{eq:rade-tilde}
{\radem}_n(\TPhF_L;\tilde S) 
\leq
2 L \paren{ \frac{\log 5k}{n} }^{1/(\beta+1)}.
\eeqn
Since by construction $h$ is $L$-Lipschitz in its first argument,
we have
\beqn
\label{eq:holder}
\abs{\sum_{i=1}^n\sigma_i\delta_i(h)} \le
\sum_{i=1}^n\abs{\delta_i(h)} \le
L\sum_{i=1}^n\dist(X_i,\tilde X_i)\le L\alpha.
\eeqn
Our claimed bound follows from (\ref{eq:rade-tilde}) and (\ref{eq:holder}).
\end{proof}

\end{document}